\documentclass[10pt,reqno]{amsart}

\usepackage[a4paper,top=0.86in,bottom=0.92in,left=0.90in,right=0.90in,
            headsep=13pt,footskip=24pt]{geometry}
\usepackage{amsmath,amssymb,amsthm,mathtools,mathrsfs}
\usepackage{booktabs,array,enumitem,float}
\usepackage{tikz,tikz-cd}
\usetikzlibrary{arrows.meta,positioning,calc,fit}
\usepackage{microtype}
\usepackage{bold-extra}
\usepackage{titlesec}
\usepackage{fancyhdr}
\usepackage[hidelinks]{hyperref}

\allowdisplaybreaks[2]
\setlist{topsep=5pt,itemsep=2.5pt,parsep=0pt,leftmargin=*}

\titleformat{\section}
  {\normalfont\large\bfseries\scshape}
  {\thesection.}{0.65em}{}
  [\vspace{0.15em}\titlerule]
\titlespacing*{\section}{0pt}{2.4ex plus .5ex minus .2ex}{1.15ex plus .2ex}

\titleformat{\subsection}
  {\normalfont\normalsize\bfseries\scshape}
  {\thesubsection.}{0.60em}{}
\titlespacing*{\subsection}{0pt}{1.75ex plus .35ex minus .15ex}{0.65ex}

\titleformat{name=\subsection,numberless}
  {\normalfont\normalsize\bfseries\scshape}
  {}{0pt}{}

\newtheoremstyle{denseplain}{8pt}{8pt}{\itshape}{}{}{.}{0.55em}{\textbf{#1 #2}\thmnote{\,\textnormal{(#3)}}}
\newtheoremstyle{densedefn}{8pt}{8pt}{\normalfont}{}{}{.}{0.55em}{\textbf{#1 #2}\thmnote{\,\textnormal{(#3)}}}
\theoremstyle{denseplain}
\newtheorem{theorem}{Theorem}[section]
\newtheorem{proposition}[theorem]{Proposition}
\newtheorem{lemma}[theorem]{Lemma}
\newtheorem{corollary}[theorem]{Corollary}
\theoremstyle{densedefn}
\newtheorem{definition}[theorem]{Definition}
\newtheorem{remark}[theorem]{Remark}
\newtheorem{example}[theorem]{Example}

\newcommand{\R}{\mathbb R}
\newcommand{\one}{\mathbf 1}
\newcommand{\Opt}{\operatorname{Opt}}
\newcommand{\argmaxa}{\operatorname*{arg\,max}}
\newcommand{\tr}{\mathsf T}
\DeclareMathOperator{\GL}{GL}

\renewcommand{\sectionmark}[1]{\markboth{#1}{}}
\fancypagestyle{plain}{%
  \fancyhf{}%
  \fancyfoot[R]{\footnotesize\thepage}%
}
\makeatletter
\let\origsetauthors\@setauthors
\def\@setauthors{%
  \origsetauthors
  \begingroup
    \centering\footnotesize\scshape
    \vspace{0\p@}%
    \affiliationline\par
  \endgroup
  \vspace{1\p@}%
}
\renewcommand{\@setaddresses}{}
\makeatother
\newcommand{\affiliationline}{Dovetail Research \;\textperiodcentered\; University of Warwick}

\begin{document}
\title[From Optimal Actions to World Models]
{From Optimal Actions to World Models\\
Identifiability of Transition Kernels in Discounted MDPs}
\author{Neal Batra}
\thanks{Funded by the Advanced Research + Invention Agency (ARIA) through
project code MSAI-SE01-P005.}
\date{}

\begin{abstract}
We study what can be recovered about the transition probabilities of a
Markov decision process from optimal actions alone. This is closely related
to the inverse problem considered by Letcher et al.~\cite{letcher2026}, who
ask when the dynamics can be recovered from numerical \(Q\)-values. Here the
numerical values themselves are not observed; only the optimal actions are
known, for every reward in a given class.

For state--action rewards \(r(s,a)\), knowing the optimal actions for every
reward also tells us how much better one action is than another when each is
followed by the same fixed policy. This is still not enough to determine the
transition probabilities uniquely. We prove that two kernels give the same
optimal actions for every reward exactly when
\[
Q_{s,a}
=
\Bigl(P_{s,a}+\tfrac1\gamma e_s^{\tr}(L-I)\Bigr)L^{-1}
\]
for one invertible matrix \(L\) satisfying \(L\one=\one\). Near a kernel
with strictly positive entries, there is an \(n(n-1)\)-dimensional family of
different kernels with this property. The result is unchanged if we consider
only rewards having a unique optimal action at every state.

We then compare this with rewards of the forms \(r(s)\) and \(r(s,a,s')\).
Rewards that depend on the next state can usually recover the transition
kernel itself: every row at a state with at least two actions is determined,
and we describe exactly when a row at a state with one action can remain
hidden. State rewards reveal less: two kernels give the same optimal actions
exactly when every deterministic policy is optimal for the same set of
rewards. The results show how the form of the reward affects what can be
learned about the dynamics from optimal actions alone.
\end{abstract}

\maketitle
\pagestyle{fancy}

\section{Introduction}

A transition model tells us both what an agent should do and what is likely
to happen after each action. These are not the same kind of information.
Two models may recommend exactly the same actions while making different
predictions about the next state. This paper asks when that can happen, even
if the optimal actions are known for every reward in a large class.

The question is closely related to recent work of Letcher et
al.~\cite{letcher2026}, who ask when a transition kernel can be recovered
from numerical \(Q\)-values. Here, we retain only optimal actions. Since one optimal action contains far less information than a full vector of \(Q\)-values, it is not clear in advance how much of the transition kernel can still be recovered.

Let \(P\) be the transition kernel of a finite discounted Markov decision
process. For a reward \(r\), write \(\Opt_{P,r}(s)\) for the set of optimal
actions at state \(s\). We classify all kernels \(Q\) satisfying
\begin{equation}
\Opt_{P,r}(s)=\Opt_{Q,r}(s)
\qquad
\text{for every permitted reward }r\text{ and every }s\in S.
\label{eq:inverse-problem}
\end{equation}
Thus \(P\) and \(Q\) are regarded as indistinguishable when every reward in
the chosen class leads to exactly the same optimal actions.

We compare three kinds of reward:
\begin{equation}
r(s),\qquad r(s,a),\qquad r(s,a,s').
\label{eq:reward-families}
\end{equation}
The distinction matters because each form of reward can test the dynamics in
a different way. A state reward can only value the states that are visited.
A state--action reward can also distinguish actions immediately. A
transition-dependent reward can directly reward or penalize a particular
successor state.

Our main result concerns state--action rewards. Suppose that \(P\) and \(Q\)
give the same optimal actions for every reward \(r(s,a)\). We prove something
stronger than equality of the optimal choices: for every deterministic policy
and every reward, taking an action once and then following that policy has
the same advantage under both kernels. Equivalently, the difference between
the returns from any two actions is the same under \(P\) and \(Q\).

This happens exactly when there is one invertible matrix \(L\), the same for
every policy and reward, such that
\begin{equation}
Q_{s,a}
=
\Bigl(P_{s,a}+\tfrac1\gamma e_s^{\tr}(L-I)\Bigr)L^{-1},
\qquad
L\one=\one.
\label{eq:Phi-intro}
\end{equation}
The matrix \(L\) relates the value functions under the two kernels. Formula
\eqref{eq:Phi-intro} gives all kernels that cannot be distinguished using
state--action rewards. Near a kernel with strictly positive entries, these
kernels form an \(n(n-1)\)-dimensional smooth subset of a space of dimension
\begin{equation}
(n-1)\sum_{s\in S}|A(s)|.
\label{eq:intro-ambient-dimension}
\end{equation}
The proof initially uses rewards for which several actions are optimal.
Appendix~\ref{app:supplementary} shows that these rewards are not essential:
the same classification already follows from rewards having a unique optimal
action at every state.

Transition-dependent rewards give substantially more information. Because
the reward may depend on the successor state, it can be chosen to distinguish
two different transition rows directly. We prove that every row at a state
with at least two actions must therefore be identified. A row at a state with
one action can remain hidden, but only in the special case where rewarding
visits to that state never changes any optimal action. As a consequence,
once the MDP contains at least one genuine choice, almost every strictly
positive kernel is identified exactly by transition-dependent rewards.

State rewards give less information. For each deterministic policy, one can
consider the set of state rewards for which that policy is optimal. We prove
that two kernels give the same optimal actions for every state reward exactly
when all of these sets agree. This gives a finite test for equivalence,
although not the explicit description obtained for state--action rewards.

The three cases therefore satisfy
\[
P\equiv_{\mathrm{tr}}Q
\Longrightarrow
P\equiv_{\mathrm{sa}}Q
\Longrightarrow
P\equiv_{\mathrm{s}}Q,
\]
and both implications are strict. Allowing the reward to depend on more of
the transition makes it possible to distinguish more transition models.

The comparison with the existing literature is as follows. Letcher et
al.~\cite{letcher2026} retain numerical \(Q\)-values for a finite collection
of rewards, whereas we retain only the maximizing actions but allow every
reward in the chosen class. The two settings therefore make different
assumptions: their data are numerically richer, while our set of rewards is
larger. Grimm et al.~\cite{grimm2020} compare models through their Bellman
updates on selected functions and policies. Richens et
al.~\cite{richens2025} study when policies for sufficiently rich goals
contain enough information to recover a world model. Earlier notions of MDP
equivalence similarly ask which changes to a model preserve quantities
relevant to decision making~\cite{givan2003}. Work on reward invariance
instead holds the transition model fixed and asks which changes to the reward
preserve behaviour~\cite{ng1999,skalse2023,mustafin2025}. Here the reward is
held fixed while the transition kernel is varied.

\section{Discounted Markov decision processes}
\label{sec:setup}

\subsection{The model}

Let $S=\{1,\ldots,n\}$ be a finite set of labelled states. At state $s$, the available labelled actions form a finite nonempty set $A(s)$. For every state and action,
\[
P_{s,a}:=P(\cdot\mid s,a)
\]
is a row vector of transition probabilities. Thus $P_{s,a}(x)\ge0$ and $P_{s,a}\one=1$, where $\one$ is the all-ones column vector. We identify a function on $S$ with a column vector in $\R^n$, and $e_s$ denotes the $s$th standard basis column vector.

Fix a discount factor $\gamma\in(0,1)$. A stationary deterministic policy $\pi$ chooses one action $\pi(s)\in A(s)$ at each state. The matrix $P^\pi$ has row $P_{s,\pi(s)}$ at state $s$.

A state--action reward is simply a collection of real numbers $r(s,a)$, one for each available pair. Write $r^\pi(s)=r(s,\pi(s))$. The value of policy $\pi$ is
\[
V_{P,r}^{\pi}(s)
:=\mathbb E_{P,\pi}\!\left[
\sum_{t=0}^{\infty}\gamma^t r(S_t,A_t)\,\middle|\,S_0=s
\right].
\]
It satisfies
\begin{equation}
V_{P,r}^{\pi}=r^\pi+\gamma P^\pi V_{P,r}^{\pi},
\qquad
V_{P,r}^{\pi}=(I-\gamma P^\pi)^{-1}r^\pi.
\label{eq:policy-value}
\end{equation}
The inverse exists because $P^\pi$ is stochastic and $\gamma \in (0,1)$.
The optimal value is the unique vector satisfying
\begin{equation}
V_{P,r}^*(s)
=\max_{a\in A(s)}\{r(s,a)+\gamma P_{s,a}V_{P,r}^*\}.
\label{eq:bellman-optimality}
\end{equation}
We write
\begin{equation}
\Opt_{P,r}(s)
:=\argmaxa_{a\in A(s)}\{r(s,a)+\gamma P_{s,a}V_{P,r}^*\}
\label{eq:optimal-actions}
\end{equation}
for the set of optimal actions at state $s$. A deterministic policy is optimal exactly when it chooses an action in this set at every state~\cite{puterman1994}.

For later use, the advantage of action \(a\) relative to a policy
\(\pi\) is
\begin{equation}
A_{P,r}^{\pi}(s,a)
:=r(s,a)+\gamma P_{s,a}V_{P,r}^{\pi}-V_{P,r}^{\pi}(s).
\label{eq:advantage}
\end{equation}
It is the gain from taking \(a\) once and then continuing with \(\pi\),
relative to following \(\pi\) immediately.

\subsection{Equivalence under different reward models}

We use three direct notions, one for each reward model.

\begin{definition}
Let \(P\) and \(Q\) be transition kernels on the same labelled state space
\(S\), with the same labelled action sets \(A(s)\) and the same discount
factor \(\gamma\).

\begin{enumerate}[label=\textup{(\roman*)}]
    \item \(P\) and \(Q\) are \emph{state--action equivalent}, written
    \[
    P\equiv_{\mathrm{sa}}Q,
    \]
    if
    \[
    \Opt_{P,r}(s)=\Opt_{Q,r}(s)
    \]
    for every state--action reward \(r(s,a)\) and every state \(s\in S\).

    \item \(P\) and \(Q\) are \emph{state-reward equivalent}, written
    \[
    P\equiv_{\mathrm{s}}Q,
    \]
    if
    \[
    \Opt_{P,r}(s)=\Opt_{Q,r}(s)
    \]
    for every state reward \(r(s)\) and every state \(s\in S\).

    \item \(P\) and \(Q\) are \emph{transition-reward equivalent}, written
    \[
    P\equiv_{\mathrm{tr}}Q,
    \]
    if
    \[
    \Opt_{P,r}(s)=\Opt_{Q,r}(s)
    \]
    for every transition-dependent reward \(r(s,a,s')\) and every state
    \(s\in S\).
\end{enumerate}
\end{definition}

The state and action labels are fixed throughout. Thus the equalities above
are equalities of the actual subsets of \(A(s)\), not equalities up to a
relabeling of states or actions.

\section{State--action rewards}
\label{sec:state-action}

We begin with rewards of the form \(r(s,a)\). Suppose we know, for every
such reward, which actions are optimal under a transition kernel. The main
result of this section asks exactly how much of the kernel is determined by
this information.

The upcoming necessity argument uses the rewards
\[
r_v^P(s,a)=v(s)-\gamma P_{s,a}v,
\]
which prescribe an arbitrary optimal value \(v\) while making every action
optimal. Comparing the resulting policy-evaluation equations under \(P\)
and \(Q\) forces one matrix \(L\) to work for every deterministic policy.
Appendix~\ref{app:supplementary} shows that the same classification is
already determined by rewards having a unique optimal action at every state.

For any $v\in\R^n$, define
\begin{equation}
r_v^P(s,a):=v(s)-\gamma P_{s,a}v.
\label{eq:value-prescribing-reward}
\end{equation}
This reward cancels the expected continuation term. Hence
\[
r_v^P(s,a)+\gamma P_{s,a}v=v(s)
\]
for every action at every state.

\begin{lemma}
\label{lem:value-prescribing-reward}
For the reward $r_v^P$, the optimal value under $P$ is $v$, and every action is optimal at every state. Conversely, if a reward makes every action optimal under $P$ and has optimal value $v$, then it must equal $r_v^P$.
\end{lemma}

\begin{proof}
Let
\[
(Tu)(s):=\max_{a\in A(s)}\{r_v^P(s,a)+\gamma P_{s,a}u\}
\]
be the optimality operator for this reward. Substituting $u=v$ gives
\[
(Tv)(s)=\max_{a\in A(s)}v(s)=v(s),
\]
and every action attains the maximum.

It remains to know that this fixed point is the optimal value. For any $u,z\in\R^n$,
\begin{align*}
|(Tu)(s)-(Tz)(s)|
&\le \max_{a\in A(s)}\gamma|P_{s,a}(u-z)|\\
&\le \gamma\|u-z\|_\infty.
\end{align*}
Thus $T$ is a contraction with modulus $\gamma<1$ and has a unique fixed point. Therefore $V_{P,r_v^P}^*=v$.

Conversely, suppose every action is optimal and the optimal value is $v$. Then every action attains equality in \eqref{eq:bellman-optimality}, so
\[
v(s)=r(s,a)+\gamma P_{s,a}v
\]
for every $s$ and $a\in A(s)$. Rearranging gives \eqref{eq:value-prescribing-reward}.
\end{proof}

The next lemma proves that one matrix $L$ relates the values under $P$ and $Q$, and that the same matrix works for every deterministic policy.

\begin{lemma}
\label{lem:forced-coordinate-map}
Assume $P\equiv_{\mathrm{sa}}Q$. Then there is a unique invertible matrix $L$ such that $L\one=\one$ and
\begin{equation}
(I-\gamma Q^\pi)L=I-\gamma P^\pi
\label{eq:all-policy-resolvent}
\end{equation}
for every deterministic policy $\pi$. Equivalently,
\begin{equation}
Q_{s,a}L=P_{s,a}+\tfrac1\gamma e_s^{\tr}(L-I)
\label{eq:forced-row-relation}
\end{equation}
for every state $s$ and action $a\in A(s)$.
\end{lemma}

\begin{proof}
Fix $v\in\R^n$. By Lemma~\ref{lem:value-prescribing-reward} every action is optimal under $P$ for the reward $r_v^P$, so by $P\equiv_{\mathrm{sa}}Q$ every action is optimal under $Q$ as well. Write $w(v):=V_{Q,r_v^P}^*$. Every deterministic policy therefore selects an optimal action at every state under $Q$, hence is optimal, so for every such $\pi$,
\[
w(v)=r_v^{P,\pi}+\gamma Q^\pi w(v),
\qquad
r_v^{P,\pi}(s):=r_v^P(s,\pi(s)).
\]
By \eqref{eq:value-prescribing-reward}, $r_v^{P,\pi}=v-\gamma P^\pi v=(I-\gamma P^\pi)v$, so
\begin{equation}
(I-\gamma Q^\pi)w(v)=(I-\gamma P^\pi)v
\label{eq:all-policy-evaluation}
\end{equation}
for every deterministic policy $\pi$ and every $v\in\R^n$.

Fix one policy $\pi_0$ and set $L:=(I-\gamma Q^{\pi_0})^{-1}(I-\gamma P^{\pi_0})$, which is invertible as a product of invertible matrices. Taking $\pi=\pi_0$ in \eqref{eq:all-policy-evaluation} gives $w(v)=Lv$, and substituting this back into \eqref{eq:all-policy-evaluation} for arbitrary $\pi$ gives
\[
(I-\gamma Q^\pi)Lv=(I-\gamma P^\pi)v
\qquad(v\in\R^n),
\]
which is \eqref{eq:all-policy-resolvent}.

Since $P^\pi\one=Q^\pi\one=\one$, applying \eqref{eq:all-policy-resolvent} to $\one$ gives
\[
(I-\gamma Q^\pi)L\one
=(I-\gamma P^\pi)\one
=(1-\gamma)\one
=(I-\gamma Q^\pi)\one,
\]
so $L\one=\one$ by invertibility of $I-\gamma Q^\pi$.

For \eqref{eq:forced-row-relation}, fix $s\in S$ and $a\in A(s)$ and choose $\pi$ with $\pi(s)=a$. The $s$th row of \eqref{eq:all-policy-resolvent} reads
\[
e_s^{\tr}L-\gamma Q_{s,a}L=e_s^{\tr}-\gamma P_{s,a},
\]
which rearranges to $Q_{s,a}L=P_{s,a}+\tfrac1\gamma e_s^{\tr}(L-I)$.

For uniqueness, suppose $M$ also satisfies $(I-\gamma Q^\pi)M=I-\gamma P^\pi$ for every deterministic policy $\pi$. Fixing any one such policy gives $(I-\gamma Q^\pi)(L-M)=0$, so $L=M$. In particular, $L$ does not depend on the choice of $\pi_0$.
\end{proof}

Let \(\mathcal G\) be the subgroup
\[
\mathcal G:=\{L\in\GL_n(\R):L\one=\one\}.
\]
For $L\in\mathcal G$, define
\begin{equation}
\Phi_L(P)_{s,a}
:=\Bigl(P_{s,a}+\tfrac1\gamma e_s^{\tr}(L-I)\Bigr)L^{-1}.
\label{eq:Phi}
\end{equation}
Every row in \eqref{eq:Phi} sums to one because $L^{-1}\one=\one$. It is a transition kernel when all entries are nonnegative.

The transformation preserves the numbers used to compare actions, after values are changed by $L$.

\begin{lemma}
\label{lem:comparison-identity}
Let $Q=\Phi_L(P)$ for some $L\in\mathcal G$, and suppose $Q$ is a transition kernel. Then, for every state--action reward $r$, every $v\in\R^n$, and every $(s,a)$,
\begin{equation}
r(s,a)+\gamma Q_{s,a}Lv-(Lv)(s)
=r(s,a)+\gamma P_{s,a}v-v(s).
\label{eq:comparison-identity}
\end{equation}
Moreover,
\begin{equation}
(I-\gamma Q^\pi)L=I-\gamma P^\pi
\label{eq:policy-matrix-identity}
\end{equation}
for every deterministic policy $\pi$.
\end{lemma}

\begin{proof}
Multiplying \eqref{eq:Phi} by $L$ gives
\[
\gamma Q_{s,a}L=\gamma P_{s,a}+e_s^{\tr}(L-I).
\]
Therefore
\begin{align*}
r(s,a)+\gamma Q_{s,a}Lv-(Lv)(s)
&=r(s,a)+\gamma P_{s,a}v+e_s^{\tr}(L-I)v-e_s^{\tr}Lv\\
&=r(s,a)+\gamma P_{s,a}v-v(s),
\end{align*}
which proves \eqref{eq:comparison-identity}. Selecting the row corresponding to $\pi(s)$ at each state gives \eqref{eq:policy-matrix-identity}.
\end{proof}

\begin{theorem}
\label{thm:sa-main}
Let \(P\) and \(Q\) be transition kernels on the same labelled states and
actions, with the same discount factor. The following are equivalent.
\begin{enumerate}[label=\textup{(\roman*)}]
\item \(P\equiv_{\mathrm{sa}}Q\).
\item For every deterministic policy \(\pi\), every state--action reward
\(r\), every state \(s\), and every action \(a\in A(s)\),
\[
A_{Q,r}^{\pi}(s,a)=A_{P,r}^{\pi}(s,a).
\]
\item There exists \(L\in\mathcal G\) such that \(Q=\Phi_L(P)\).
\end{enumerate}
The matrix \(L\) is unique. When these conditions hold,
\begin{align}
(I-\gamma Q^\pi)L&=I-\gamma P^\pi,
\label{eq:policy-matrix-identity-main}\\
V_{Q,r}^{\pi}&=LV_{P,r}^{\pi},
\label{eq:policy-value-transform}\\
V_{Q,r}^*&=LV_{P,r}^*.
\label{eq:optimal-value-transform}
\end{align}
\end{theorem}

\begin{proof}
Assume first that \textup{(i)} holds. Lemma~\ref{lem:forced-coordinate-map}
gives a unique \(L\in\mathcal G\) satisfying
\[
(I-\gamma Q^\pi)L=I-\gamma P^\pi
\]
for every deterministic policy \(\pi\). Reading this identity row by row
gives \(Q=\Phi_L(P)\), so \textup{(iii)} holds.

Now assume \textup{(iii)}. Lemma~\ref{lem:comparison-identity} gives
\eqref{eq:policy-matrix-identity-main}. For a fixed policy \(\pi\) and
reward \(r\),
\[
(I-\gamma Q^\pi)LV_{P,r}^{\pi}
=(I-\gamma P^\pi)V_{P,r}^{\pi}
=r^\pi.
\]
The solution of \((I-\gamma Q^\pi)u=r^\pi\) is unique, so
\eqref{eq:policy-value-transform} holds. Applying
\eqref{eq:comparison-identity} with \(v=V_{P,r}^{\pi}\) gives
\[
A_{Q,r}^{\pi}(s,a)=A_{P,r}^{\pi}(s,a),
\]
which proves \textup{(ii)}.

It remains to prove that \textup{(ii)} implies \textup{(i)}. Fix a reward
\(r\), a state \(s\), and an action \(a\in\Opt_{P,r}(s)\). Choose a
deterministic policy \(\pi\) with \(\pi(s)=a\) and
\(\pi(t)\in\Opt_{P,r}(t)\) at every other state \(t\). This policy is
optimal under \(P\), so
\[
A_{P,r}^{\pi}(t,c)\le0
\qquad(t\in S,\ c\in A(t)).
\]
By \textup{(ii)}, the same inequalities hold under \(Q\), so \(\pi\) is
optimal under \(Q\). In particular, \(a=\pi(s)\) is optimal at \(s\) under
\(Q\). Hence \(\Opt_{P,r}(s)\subseteq\Opt_{Q,r}(s)\). Interchanging \(P\)
and \(Q\) gives the reverse inclusion, proving \textup{(i)}.

Finally, let \(v=V_{P,r}^*\) and \(w=Lv\). By
\eqref{eq:comparison-identity},
\[
r(s,a)+\gamma Q_{s,a}w-w(s)
=
r(s,a)+\gamma P_{s,a}v-v(s).
\]
The right-hand side is nonpositive for every action and is zero for at least
one action at each state. Thus \(w\) satisfies the optimality equation under
\(Q\), proving \eqref{eq:optimal-value-transform}.
\end{proof}

\begin{corollary}
\label{cor:orbit}
The kernels that are state--action equivalent to $P$ are exactly
\begin{equation}
[P]_{\mathrm{sa}}
=
\{\Phi_L(P):L\in\mathcal G,\ \Phi_L(P)\text{ is a transition kernel}\}.
\label{eq:orbit-class}
\end{equation}
The matrix $L$ is unique. Given $P$ and $Q$, it may be recovered from any deterministic policy $\pi$ by
\begin{equation}
L=(I-\gamma Q^\pi)^{-1}(I-\gamma P^\pi).
\label{eq:L-recovery}
\end{equation}
Thus membership in the equivalence class can be checked by computing this $L$, verifying $L\one=\one$, and checking \eqref{eq:forced-row-relation} for every state and action.
\end{corollary}

\begin{proof}
The set identity is Theorem~\ref{thm:sa-main}. Rearranging \eqref{eq:policy-matrix-identity-main} gives \eqref{eq:L-recovery}. Uniqueness was proved in Lemma~\ref{lem:forced-coordinate-map}.
\end{proof}

As $L$ varies, the set in \eqref{eq:orbit-class} is the orbit of $P$ under the transformations $\Phi_L$; the word ``orbit'' refers only to this set of transformed kernels.

The theorem also gives two useful identities not involving a reward. If $a,b\in A(s)$, then
\begin{equation}
(Q_{s,a}-Q_{s,b})L=P_{s,a}-P_{s,b}.
\label{eq:row-difference-transform}
\end{equation}
For every deterministic policy $\pi$,
\begin{equation}
\frac{\det(I-\gamma P^\pi)}{\det(I-\gamma Q^\pi)}=\det L.
\label{eq:determinant-invariant}
\end{equation}
In particular, the determinant ratio is independent of the policy.
\subsection{Local geometry of the equivalence class}

Each transition row \(P_{s,a}\) has \(n\) entries subject to the single constraint \[ P_{s,a}\one=1. \] Thus each row has \(n-1\) free parameters. Since there is one transition row for every available state--action pair, the ambient space of transition kernels with the fixed states and actions has dimension
\begin{equation}
D=(n-1)\sum_{s\in S}|A(s)|.
\label{eq:ambient-dimension}
\end{equation}
The next result describes the local structure of
\([P]_{\mathrm{sa}}\): its dimension and the directions in which \(P\)
may be varied while remaining in the same equivalence class.

\begin{proposition}
\label{prop:orbit-geometry}
Suppose every transition probability \(P(x\mid s,a)\) is strictly positive.
Then, in a neighbourhood of \(P\), the equivalence class
\([P]_{\mathrm{sa}}\) is a smooth embedded manifold of dimension
\[
n(n-1).
\]
Its tangent space at \(P\) is
\begin{equation}
T_P[P]_{\mathrm{sa}}
=
\left\{
\dot P:
\dot P_{s,a}
=
\frac{1}{\gamma}e_s^{\tr}H-P_{s,a}H,
\quad
H\one=0
\right\}.
\label{eq:orbit-tangent-space}
\end{equation}
\end{proposition}

\begin{proof}
The condition \(L\one=\one\) imposes \(n\) independent linear equations
on the \(n^2\) entries of \(L\). Hence \(\mathcal G\) is a smooth manifold
of dimension
\[
n^2-n=n(n-1).
\]

Consider the map
\[
F_P(L):=\Phi_L(P).
\]
It is continuous and satisfies \(F_P(I)=P\). Since every entry of \(P\)
is positive, sufficiently small changes of \(L\) leave every entry of
\(F_P(L)\) positive. Thus, near the identity, \(F_P\) takes values in the
space of transition kernels.

The map is injective near the identity, and in fact globally wherever its
image is a transition kernel. If
\[
F_P(L)=F_P(M)=Q,
\]
then, for any fixed deterministic policy \(\pi\),
\[
(I-\gamma Q^\pi)L
=
I-\gamma P^\pi
=
(I-\gamma Q^\pi)M.
\]
Since \(I-\gamma Q^\pi\) is invertible, \(L=M\). Moreover, the inverse on
the image is given by the smooth formula
\[
Q\longmapsto
(I-\gamma Q^\pi)^{-1}(I-\gamma P^\pi).
\]
Therefore the image near \(P\) is a smooth embedded manifold with the same
dimension as \(\mathcal G\).

To compute its tangent space, let
\[
L(\varepsilon)
=
I+\varepsilon H+O(\varepsilon^2),
\qquad
H\one=0.
\]
Then
\[
L(\varepsilon)^{-1}
=
I-\varepsilon H+O(\varepsilon^2).
\]
Substituting into \eqref{eq:Phi} gives
\begin{align*}
\Phi_{L(\varepsilon)}(P)_{s,a}
&=
\left(
P_{s,a}
+
\frac{\varepsilon}{\gamma}e_s^{\tr}H
+
O(\varepsilon^2)
\right)
\left(
I-\varepsilon H+O(\varepsilon^2)
\right)\\
&=
P_{s,a}
+
\varepsilon
\left(
\frac{1}{\gamma}e_s^{\tr}H-P_{s,a}H
\right)
+
O(\varepsilon^2).
\end{align*}
The coefficient of \(\varepsilon\) gives the tangent vectors in
\eqref{eq:orbit-tangent-space}.
\end{proof}

\begin{corollary}
\label{cor:equal-action-dimension}
Suppose the assumptions of Proposition~\ref{prop:orbit-geometry} hold and
every state has exactly \(m\) actions. Then the ambient kernel space has
dimension
\[
mn(n-1),
\]
whereas \([P]_{\mathrm{sa}}\) has local dimension
\[
n(n-1).
\]
\end{corollary}

\begin{proof}
When \(|A(s)|=m\) for every \(s\), equation
\eqref{eq:ambient-dimension} gives
\[
D=(n-1)\sum_{s\in S}m
=mn(n-1).
\]
The remaining statements follow from
Proposition~\ref{prop:orbit-geometry}.
\end{proof}

\begin{remark}
Corollary~\ref{cor:equal-action-dimension} shows that the ratio of the local
dimension of \([P]_{\mathrm{sa}}\) to the dimension of the ambient kernel
space is
\[
\frac{n(n-1)}{mn(n-1)}=\frac{1}{m}.
\]
Thus state--action equivalence becomes increasingly restrictive as the number
of available actions grows. Each additional action introduces further action
comparisons that must be preserved for two kernels to remain in the same
equivalence class. Consequently, although distinct kernels may still produce
exactly the same optimal-action data, the equivalence class becomes
progressively smaller relative to the ambient kernel space as \(m\) increases.
This is a local geometric statement; it does not mean that the observations
retain a literal fraction \(1/m\) of the statistical information about the
kernel.
\end{remark}

\begin{example}
\label{ex:nontrivial-pair}
Let $S=\{0,1\}$, let both states have actions $a,b$, and take $\gamma=4/5$. Set
\[
L=\begin{pmatrix}11/10&-1/10\\1/5&4/5\end{pmatrix},
\qquad L\one=\one.
\]
Define $P$ and $Q$ by
\[
\begin{array}{c|cc}
& a&b\\ \hline
P(\cdot\mid0,\cdot)&(13/20,7/20)&(7/20,13/20)\\
P(\cdot\mid1,\cdot)&(11/20,9/20)&(1/4,3/4)
\end{array}
\qquad
\begin{array}{c|cc}
& a&b\\ \hline
Q(\cdot\mid0,\cdot)&(23/36,13/36)&(11/36,25/36)\\
Q(\cdot\mid1,\cdot)&(2/3,1/3)&(1/3,2/3).
\end{array}
\]
Direct substitution gives $Q=\Phi_L(P)$. Hence $P\ne Q$, but the two kernels have the same optimal actions for every state--action reward.
\end{example}

\section{Transition-dependent rewards}
\label{sec:transition-rewards}

A transition-dependent reward assigns a number \(r(s,a,x)\) when action
\(a\) at state \(s\) leads to the next state \(x\). Under a kernel \(P\),
it induces the state--action reward
\begin{equation}
\bar r_P(s,a)
:=
\sum_{x\in S}P(x\mid s,a)r(s,a,x).
\label{eq:expected-transition-reward}
\end{equation}
Thus the same transition-dependent reward \(r\) may induce different
state--action rewards under different transition kernels.

This makes an individual transition row easy to test when a second action
is available at the same state. If \(P_{s,a}\ne Q_{s,a}\), one can choose
\(r(s,a,\cdot)\) to have expectation zero under one row and nonzero under
the other, and compare action \(a\) with another action at \(s\). The only
case requiring care is a state with one action. There is no comparison at
that state, although its row may still affect choices elsewhere through the
continuation value.

We use the following elementary fact.

\begin{lemma}
\label{lem:irrelevant-state-rewards}
For a fixed kernel \(K\), the state rewards for which every action is optimal
at every state form a vector subspace of \(\R^n\). If \(z\) is such a reward,
then for every state--action reward \(r\),
\[
\Opt_{K,r+z}(s)=\Opt_{K,r}(s)
\qquad(s\in S),
\]
where \((r+z)(s,a)=r(s,a)+z(s)\).
\end{lemma}

\begin{proof}
Suppose every action is optimal for state rewards \(z_1,\ldots,z_m\), and
write \(w_i=V^*_{K,z_i}\). Then
\[
w_i(s)=z_i(s)+\gamma K_{s,a}w_i
\qquad(s\in S,\ a\in A(s)).
\]
For scalars \(c_1,\ldots,c_m\), the vectors
\[
z=\sum_i c_i z_i,
\qquad
w=\sum_i c_i w_i
\]
satisfy the same identity, so every action is optimal for \(z\). This proves
the vector-space statement.

Now let \(v=V^*_{K,r}\). For every state and action,
\[
r(s,a)+z(s)+\gamma K_{s,a}(v+w)
=
r(s,a)+\gamma K_{s,a}v+w(s).
\]
The maximizing actions are therefore unchanged, and \(v+w\) is the optimal
value for \(r+z\).
\end{proof}

Let
\[
U := \{s\in S : |A(s)|=1\},
\]
and write $a_u$ for the unique action at $u\in U$. For $u\in U$, let $e_u$ denote the state reward
\[
e_u(s)=
\begin{cases}
1, & s=u,\\
0, & s\ne u.
\end{cases}
\]
Thus $e_u$ rewards each visit to $u$.
\begin{theorem}
\label{thm:transition-classification}
Let \(P\) and \(Q\) be transition kernels on the same labelled states and
actions. Then \(P\equiv_{\mathrm{tr}}Q\) if and only if:
\begin{enumerate}[label=\textup{(\roman*)}]
\item \(P\equiv_{\mathrm{sa}}Q\);
\item \(P_{s,a}=Q_{s,a}\) whenever \(|A(s)|\ge2\);
\item if \(u\in U\) and \(P_{u,a_u}\ne Q_{u,a_u}\), then every action is
optimal at every state under \(P\) for the state reward \(e_u\).
\end{enumerate}
Under \textup{(i)}, the last condition is unchanged if \(P\) is replaced by
\(Q\).
\end{theorem}

\begin{proof}
Assume first that \(P\equiv_{\mathrm{tr}}Q\). Every state--action reward
\(\rho(s,a)\) is the transition reward \(r(s,a,s')=\rho(s,a)\), so
\(P\equiv_{\mathrm{sa}}Q\).

Suppose that \(P_{s,a}\ne Q_{s,a}\) at a state with at least two actions.
Write \(p=P_{s,a}\) and \(q=Q_{s,a}\). Since \(p\) and \(q\) are distinct
probability rows, there is an \(h\in\R^n\) such that
\[
ph=0,
\qquad
qh=1.
\]
Choose \(b\in A(s)\setminus\{a\}\), and set
\[
r(s,a,x)=h(x),
\qquad
r(t,c,x)=0\quad\text{when }(t,c)\ne(s,a).
\]
Under \(P\), the expected immediate reward is zero for every state and
action, so every action is optimal. Under \(Q\), a policy choosing \(b\) at
\(s\) receives zero reward everywhere and therefore has value zero. If every
action were optimal under \(Q\), then this policy would be optimal. However,
at state \(s\), taking action \(a\) gives immediate expected reward \(1\),
whereas taking action \(b\) gives immediate expected reward \(0\). Hence
\(a\) is strictly better than \(b\) at \(s\), a contradiction. Therefore the
rows must agree at every state with at least two actions.

Now let \(u\in U\) and suppose
\(P_{u,a_u}\ne Q_{u,a_u}\). Choose \(h\) as above and support \(r\) only on
\((u,a_u)\). Its expected reward is zero under \(P\) and is the state reward
\(e_u\) under \(Q\). Transition-reward equivalence therefore implies that
every action is optimal under \(Q\) for \(e_u\). Since
\(P\equiv_{\mathrm{sa}}Q\), the same is true under \(P\).

Conversely, assume \textup{(i)}--\textup{(iii)}, and fix a transition reward
\(r\). Put
\[
\rho(s,a)=\bar r_P(s,a),
\qquad
d(s,a)=\bar r_Q(s,a)-\bar r_P(s,a).
\]
By \textup{(ii)}, \(d(s,a)=0\) at every state with at least two actions. It
also vanishes at a one-action state whose row is the same under \(P\) and
\(Q\). Hence
\[
d=\sum_{u\in D}c_u e_u,
\qquad
D:=\{u\in U:P_{u,a_u}\ne Q_{u,a_u}\},
\]
for suitable scalars \(c_u\), where state rewards are identified with
state--action rewards constant across the available actions. By \textup{(i)}
and \textup{(iii)}, every action is optimal under \(Q\) for each \(e_u\)
with \(u\in D\). Lemma~\ref{lem:irrelevant-state-rewards} therefore gives
\[
\Opt_{Q,\rho+d}=\Opt_{Q,\rho}.
\]
Using \(P\equiv_{\mathrm{sa}}Q\),
\[
\Opt_{P,r}
=
\Opt_{P,\rho}
=
\Opt_{Q,\rho}
=
\Opt_{Q,\rho+d}
=
\Opt_{Q,r}.
\]
Thus \(P\equiv_{\mathrm{tr}}Q\).
\end{proof}

Condition \textup{(iii)} can be checked directly. Fix any deterministic
policy \(\pi\), and for \(u\in U\) set
\begin{equation}
h_u:=(I-\gamma P^\pi)^{-1}e_u.
\label{eq:singleton-visit-vector}
\end{equation}
Then every action is optimal for the state reward \(e_u\) if and only if
\begin{equation}
(P_{s,a}-P_{s,b})h_u=0
\qquad
(s\in S,\ a,b\in A(s)).
\label{eq:singleton-row-test}
\end{equation}
Indeed, \(h_u\) is the value of \(\pi\) for the reward \(e_u\), and the
displayed condition says that replacing \(\pi(s)\) by any other action leaves
the value unchanged.

\begin{corollary}
\label{cor:transition-identification}
Suppose that for every \(u\in U\), the state reward \(e_u\) makes some action
non-optimal at some state. Then transition-dependent rewards determine
\(P\) exactly. In particular, exact identification holds when every state
has at least two actions.
\end{corollary}

\begin{proof}
Let \(Q\equiv_{\mathrm{tr}}P\). By
Theorem~\ref{thm:transition-classification}\textup{(ii)}, \(P\) and \(Q\)
agree at every state with at least two actions. If they differed at some
\(u\in U\), condition \textup{(iii)} would imply that every action is
optimal for the reward \(e_u\), contrary to the hypothesis. Hence every row
agrees and \(P=Q\). When every state has at least two actions, \(U\) is empty.
\end{proof}

\begin{proposition}
    \label{prop:generic-transition-identification}
Suppose \(n\ge2\) and at least one state has at least two actions. Among
strictly positive transition kernels, the kernels not identified by all
transition-dependent rewards lie in a set of Lebesgue measure zero in the
usual coordinates on the kernel space.
\end{proposition}

\begin{proof}
Fix a state \(s_0\) with distinct actions \(a,b\), and fix a deterministic
policy \(\pi\) with \(\pi(s_0)=b\). For a singleton-action state \(u\), set
\[
f_u(P)
:=
(P_{s_0,a}-P_{s_0,b})(I-\gamma P^\pi)^{-1}e_u.
\]
If the reward \(e_u\) makes every action optimal, then
\(f_u(P)=0\) by \eqref{eq:singleton-row-test}. Multiplying by \(\det(I-\gamma P^\pi)\) writes \(f_u\) as a
polynomial divided by a nonzero denominator.

The numerator polynomial is not identically zero. To see this, choose \(P\)
so that every row of \(P^\pi\) is uniform on \(S\). In particular,
\(P_{s_0,b}\) is the uniform row. For some \(v\ne u\), take
\[
P_{s_0,a}=P_{s_0,b}+\delta(e_u^{\tr}-e_v^{\tr})
\]
with \(\delta>0\) small. Then the kernel is strictly positive and
\[
(I-\gamma P^\pi)^{-1}e_u
=
e_u+\frac{\gamma}{n(1-\gamma)}\one,
\]
so \(f_u(P)=\delta\ne0\). Hence the condition \(f_u(P)=0\) holds on a
measure-zero set.

If \(P\) is not identified by all transition-dependent rewards, then by
Corollary~\ref{cor:transition-identification} there is some \(u\in U\) for
which \(e_u\) makes every action optimal. For this \(u\),
\eqref{eq:singleton-row-test} implies \(f_u(P)=0\). Hence the kernels that
are not identified lie in the finite union
\[
\bigcup_{u\in U}\{P:f_u(P)=0\}.
\]
Each set in this union has measure zero, so the union has measure zero.
\end{proof}

\begin{example}
\label{ex:hidden-singleton-row}
Let
\[
S=\{0,1\},
\qquad
A(0)=\{a,b\},
\qquad
A(1)=\{c\},
\]
and define
\[
P_{0,a}=P_{0,b}=Q_{0,a}=Q_{0,b}
=\left(\frac12,\frac12\right),
\]
while
\[
P_{1,c}=(1,0),
\qquad
Q_{1,c}=\left(\frac12,\frac12\right).
\]
\end{example}

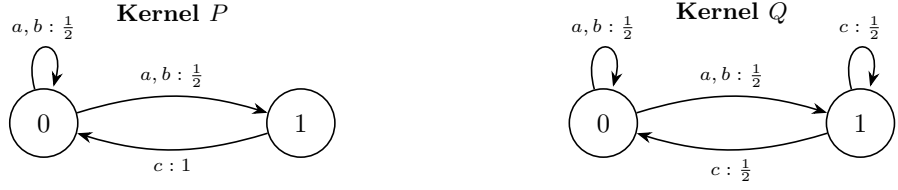
\begin{figure}[H]
\centering
\begin{tikzpicture}[
    >=Stealth,
    state/.style={
        circle,
        draw,
        semithick,
        minimum size=9mm,
        inner sep=0pt
    },
    every edge/.style={
        draw,
        ->,
        semithick
    },
    every edge quotes/.style={
        font=\scriptsize,
        fill=white,
        inner sep=1.5pt
    }
]

\begin{scope}[xshift=0cm]
    \node[font=\small\bfseries] at (1.7,1.45) {Kernel \(P\)};

    \node[state] (P0) at (0,0) {\(0\)};
    \node[state] (P1) at (3.4,0) {\(1\)};

    \path
        (P0) edge[loop above]
        node[above,font=\scriptsize]
        {\(a,b:\tfrac12\)}
        (P0)

        (P0) edge[bend left=17]
        node[above,font=\scriptsize]
        {\(a,b:\tfrac12\)}
        (P1)

        (P1) edge[bend left=17]
        node[below,font=\scriptsize]
        {\(c:1\)}
        (P0);
\end{scope}

\begin{scope}[xshift=7.4cm]
    \node[font=\small\bfseries] at (1.7,1.45) {Kernel \(Q\)};

    \node[state] (Q0) at (0,0) {\(0\)};
    \node[state] (Q1) at (3.4,0) {\(1\)};

    \path
        (Q0) edge[loop above]
        node[above,font=\scriptsize]
        {\(a,b:\tfrac12\)}
        (Q0)

        (Q0) edge[bend left=17]
        node[above,font=\scriptsize]
        {\(a,b:\tfrac12\)}
        (Q1)

        (Q1) edge[bend left=17]
        node[below,font=\scriptsize]
        {\(c:\tfrac12\)}
        (Q0)

        (Q1) edge[loop above]
        node[above,font=\scriptsize]
        {\(c:\tfrac12\)}
        (Q1);
\end{scope}

\end{tikzpicture}
\caption{The kernels in Example~\ref{ex:hidden-singleton-row}. At state
\(0\), actions \(a\) and \(b\) have the same transition row under both
kernels. The only changed row is attached to the unique action \(c\) at
state \(1\), so the change cannot affect which actions are optimal.}
\label{fig:hidden-singleton-row}
\end{figure}

Thus \(P\) and \(Q\) have different transition probabilities but always give
the same optimal actions. The difference is hidden because state \(1\) has
only one possible action, while the two actions at state \(0\) behave in
exactly the same way.

\section{State rewards}
\label{sec:state-rewards}

We now consider rewards of the form
\[
r:S\to\R,
\]
so the reward depends only on the current state. The optimal value
function satisfies
\[
V^*_{P,r}(s)
=
r(s)
+
\gamma\max_{a\in A(s)}P_{s,a}V^*_{P,r}.
\]

It is enough to consider deterministic stationary policies. Indeed,
every finite discounted MDP has an optimal policy of this form, and a
deterministic stationary policy is optimal exactly when it chooses an
optimal action at every state.

For a deterministic stationary policy \(\pi\), let \(P^\pi\) be the
transition matrix whose \(s\)-th row is \(P_{s,\pi(s)}\). Its value
function is
\[
V^\pi_{P,r}
=
(I-\gamma P^\pi)^{-1}r.
\]

Define
\[
\mathcal R_P(\pi)
:=
\left\{
r\in\R^n:
(P_{s,\pi(s)}-P_{s,a})
(I-\gamma P^\pi)^{-1}r
\geq0
\text{ for every }s\in S,\ a\in A(s)
\right\}.
\]
Thus \(\mathcal R_P(\pi)\) is the set of state rewards for which \(\pi\) is optimal under \(P\). Such policy regions for discounted MDPs go back to Smallwood~\cite{smallwood1966}.

\begin{lemma}
\label{lem:state-policy-region}
A deterministic stationary policy \(\pi\) is optimal under the state
reward \(r\) if and only if
\[
r\in\mathcal R_P(\pi).
\]
\end{lemma}

\begin{proof}
The policy \(\pi\) is optimal if and only if, at every state, its
chosen action is at least as good as every other action. Thus
\[
r(s)+\gamma P_{s,\pi(s)}V^\pi_{P,r}
\geq
r(s)+\gamma P_{s,a}V^\pi_{P,r}
\]
for every \(s\in S\) and \(a\in A(s)\). Cancelling \(r(s)\) and using
\[
V^\pi_{P,r}=(I-\gamma P^\pi)^{-1}r
\]
gives
\[
(P_{s,\pi(s)}-P_{s,a})
(I-\gamma P^\pi)^{-1}r
\geq0.
\]
These are precisely the inequalities defining
\(\mathcal R_P(\pi)\).
\end{proof}

\begin{theorem}
\label{thm:state-reward-classification}
Let \(P\) and \(Q\) be transition kernels on the same labelled state
and action spaces. Then
\[
P\equiv_{\mathrm{s}}Q
\]
if and only if
\[
\mathcal R_P(\pi)=\mathcal R_Q(\pi)
\]
for every deterministic stationary policy \(\pi\).
\end{theorem}

\begin{proof}
Suppose first that \(P\equiv_{\mathrm{s}}Q\). A deterministic stationary policy
is optimal exactly when it chooses an optimal action at every state.
Since \(P\) and \(Q\) have the same optimal-action sets for every
state reward, they have the same optimal deterministic stationary
policies. Lemma~\ref{lem:state-policy-region} therefore gives
\[
\mathcal R_P(\pi)=\mathcal R_Q(\pi)
\]
for every \(\pi\).

Conversely, suppose that
\[
\mathcal R_P(\pi)=\mathcal R_Q(\pi)
\]
for every deterministic stationary policy \(\pi\). By
Lemma~\ref{lem:state-policy-region}, \(P\) and \(Q\) have the same
optimal deterministic stationary policies for every state reward.

Fix \(r\in\R^n\), \(s\in S\), and \(a\in A(s)\). The action \(a\) is
optimal at \(s\) under \(P\) if and only if there is an optimal
deterministic stationary policy \(\pi\) such that
\[
\pi(s)=a.
\]
Indeed, one may choose \(a\) at \(s\) and choose any optimal action at
each other state. Since the optimal deterministic stationary policies
are the same under \(P\) and \(Q\), it follows that \(a\) is optimal
at \(s\) under \(P\) if and only if it is optimal at \(s\) under
\(Q\). Hence
\[
\Opt_{P,r}(s)=\Opt_{Q,r}(s)
\]
for every \(r\) and \(s\), so
\[
P\equiv_{\mathrm{s}}Q.
\]
\end{proof}

The theorem gives a finite classification. There are only finitely
many deterministic stationary policies, and for each policy \(\pi\),
the set \(\mathcal R_P(\pi)\) is determined by finitely many linear
inequalities in \(r\).

\begin{corollary}
\label{cor:state-action-implies-state}
If
\[
P\equiv_{\mathrm{sa}}Q,
\]
then
\[
P\equiv_{\mathrm{s}}Q.
\]
\end{corollary}

\begin{proof}
Every state reward \(r:S\to\R\) can be regarded as the state--action
reward
\[
\widetilde r(s,a):=r(s).
\]
Thus equality of the optimal-action sets for every state--action
reward implies equality for every state reward.
\end{proof}

The converse does not hold.

\begin{example}
\label{ex:state-strictly-weaker}
State-reward equivalence is strictly weaker than state--action
equivalence, even with two states, two actions, and strictly positive
transition probabilities.

Let
\[
S=\{0,1\},
\qquad
A(0)=A(1)=\{a,b\},
\]
and define
\[
P^a
=
\frac1{10}
\begin{pmatrix}
8&2\\
7&3
\end{pmatrix},
\qquad
P^b
=
\frac1{10}
\begin{pmatrix}
4&6\\
5&5
\end{pmatrix},
\]
and
\[
Q^a
=
\frac1{10}
\begin{pmatrix}
6&4\\
8&2
\end{pmatrix},
\qquad
Q^b
=
\frac1{10}
\begin{pmatrix}
5&5\\
5&5
\end{pmatrix}.
\]

We first show that
\[
P\equiv_{\mathrm{s}}Q.
\]
Let \(R\) denote either \(P\) or \(Q\), and let
\[
d_R:=V^*_{R,r}(1)-V^*_{R,r}(0).
\]
For both kernels and at both states, action \(b\) has a larger
probability of moving to state \(1\) than action \(a\). Therefore,
if \(d_R>0\), action \(b\) is uniquely optimal at both states; if
\(d_R<0\), action \(a\) is uniquely optimal at both states; and if
\(d_R=0\), both actions are optimal.

Suppose that \(d_R>0\). Since action \(b\) is then optimal at both
states, the optimality identities at states \(1\) and \(0\) give
\[
d_R
=
r(1)-r(0)
+
\gamma
\bigl(
R(1\mid1,b)-R(1\mid0,b)
\bigr)d_R.
\]
Hence
\[
d_R
=
\frac{r(1)-r(0)}
{1-\gamma
\bigl(
R(1\mid1,b)-R(1\mid0,b)
\bigr)}.
\]
The denominator is strictly positive, so \(d_R>0\) implies
\[
r(1)>r(0).
\]

Similarly, if \(d_R<0\), action \(a\) is optimal at both states, and
the optimality identities give
\[
d_R
=
\frac{r(1)-r(0)}
{1-\gamma
\bigl(
R(1\mid1,a)-R(1\mid0,a)
\bigr)}.
\]
Again the denominator is strictly positive, so \(d_R<0\) implies
\[
r(1)<r(0).
\]
Finally, if \(d_R=0\), the optimality identities at the two states give
\[
r(1)=r(0).
\]
It follows that
\[
\operatorname{sgn}(d_R)
=
\operatorname{sgn}\bigl(r(1)-r(0)\bigr).
\]

Consequently, for both \(P\) and \(Q\),
\[
\Opt_{R,r}(s)
=
\begin{cases}
\{b\},&r(1)>r(0),\\[2mm]
\{a\},&r(1)<r(0),\\[2mm]
\{a,b\},&r(1)=r(0),
\end{cases}
\]
at both states. Hence
\[
P\equiv_{\mathrm{s}}Q.
\]

We now show that
\[
P\not\equiv_{\mathrm{sa}}Q.
\]
If \(P\equiv_{\mathrm{sa}}Q\), then by the state--action
classification there would be an invertible matrix \(L\) satisfying
\[
L\one=\one
\]
and
\[
(Q_{s,b}-Q_{s,a})L
=
P_{s,b}-P_{s,a}
\]
for both states \(s\).

Let
\[
u=(-1,1).
\]
The action differences are
\[
P_{0,b}-P_{0,a}
=
\frac4{10}u,
\qquad
P_{1,b}-P_{1,a}
=
\frac2{10}u,
\]
whereas
\[
Q_{0,b}-Q_{0,a}
=
\frac1{10}u,
\qquad
Q_{1,b}-Q_{1,a}
=
\frac3{10}u.
\]

Since \(L\one=\one\),
\[
(uL)\one=u\one=0.
\]
The zero-sum subspace of \(\R^2\) is spanned by \(u\), so
\[
uL=\lambda u
\]
for some \(\lambda\in\R\). The identity at state \(0\) would then give
\[
\frac1{10}\lambda u
=
\frac4{10}u,
\]
and hence
\[
\lambda=4.
\]
The identity at state \(1\) would give
\[
\frac3{10}\lambda u
=
\frac2{10}u,
\]
and hence
\[
\lambda=\frac23.
\]
This is impossible. Therefore
\[
P\not\equiv_{\mathrm{sa}}Q.
\]

Thus
\[
P\equiv_{\mathrm{s}}Q
\qquad\text{but}\qquad
P\not\equiv_{\mathrm{sa}}Q.
\]
\end{example}

\begin{corollary}
\label{cor:strict-hierarchy}
The three notions of equivalence satisfy
\[
P\equiv_{\mathrm{tr}}Q
\quad\Longrightarrow\quad
P\equiv_{\mathrm{sa}}Q
\quad\Longrightarrow\quad
P\equiv_{\mathrm{s}}Q,
\]
and neither implication can be reversed.
\end{corollary}

\begin{proof}
The first implication follows from
Theorem~\ref{thm:transition-classification}, and the second is
Corollary~\ref{cor:state-action-implies-state}.
Example~\ref{ex:nontrivial-pair} shows that state--action equivalence
does not imply transition-reward equivalence, while
Example~\ref{ex:state-strictly-weaker} shows that state-reward
equivalence does not imply state--action equivalence.
\end{proof}

\section{Limitations and open problems}

\subsection*{All rewards are still required}

Corollary~\ref{cor:strict-rewards-suffice} shows that rewards with
several optimal actions are not needed, but the classification still assumes access to the unique optimal
actions over all state--action rewards. Ordinary data may contain only one
selected action on visited states and only a finite collection of tasks.
Determining how much of the equivalence class can be removed by a finite or
adaptively chosen set of rewards remains open.

\subsection*{Restricted reward classes}

If rewards are restricted to
\[
r_\theta(s,a)=\langle\phi(s,a),\theta\rangle,
\]
the remaining ambiguity should depend on which policy action gaps can be
generated by the feature vectors. A finite-reward or rank condition would
connect the exact classification here with the design of informative tasks.

\subsection*{Approximate observations}

The equivalence notions in this paper are exact. A quantitative extension
would ask how close a kernel must be to the family
\eqref{eq:orbit-class} when action gaps or optimal-action regions agree only
approximately. This would turn the classification into stability bounds for
noisy values, finite reward samples, or approximately optimal actions.

\subsection*{A direct parameterization for state rewards}

Theorem~\ref{thm:state-reward-classification} gives a finite exact test for
state-reward equivalence, but not a direct parameterization of all equivalent
kernels. Finding a description comparable to \eqref{eq:orbit-class} remains
open.

\fancyhead[L]{\footnotesize\scshape Conclusion}
\section{Conclusion}

Knowing which actions are optimal does not usually determine the transition
kernel. Nevertheless, when the optimal actions are known for every reward,
they reveal more than a single choice might suggest.

For state--action rewards, they determine every comparison between actions
under every deterministic policy. Two kernels can therefore agree on all
optimal actions only if they also agree on how much better one action is than
another. Even this does not determine the kernel itself. The remaining
kernels are described exactly by the matrix transformation
\eqref{eq:Phi-intro}, and near a strictly positive kernel they form a family
of dimension \(n(n-1)\). Rewards with a unique optimal action at every state
already determine this same family.

The amount that can be recovered changes with the reward. State rewards may
leave substantially different kernels indistinguishable. State--action
rewards determine all action comparisons but still leave a continuous
ambiguity in the dynamics. Transition-dependent rewards can test successor
states directly and, except for a precisely described obstruction at states
with one action, recover the transition kernel itself.

The results separate two questions that are often treated as one. A model
may contain enough information to choose the same actions for a given class
of rewards without giving the correct probabilities of what happens next.
That may be sufficient when the only aim is to reproduce those decisions.
It is not sufficient for predicting transitions, simulating the process,
evaluating counterfactual actions, or transferring to rewards outside the
observed class.
\clearpage
\appendix
\fancyhead[L]{\footnotesize\scshape Appendix: Supplementary Results}

\section{Supplementary results}
\label{app:supplementary}

\subsection{Unique optimal actions suffice}

The purpose of this subsection is to show that rewards with several optimal
actions are not needed for the state--action classification. It is enough to
know the uniquely optimal action at every state for every reward for which
such an action exists. The reason is geometric.

For a deterministic policy \(\pi\), let
\[
\mathcal U_P(\pi)
:=
\left\{
r:
\Opt_{P,r}(s)=\{\pi(s)\}
\text{ for every }s\in S
\right\}.
\]
Thus \(\mathcal U_P(\pi)\) is the set of state--action rewards for which
\(\pi\) is the unique optimal deterministic policy.

\begin{lemma}
\label{lem:strict-region-closure}
The closure of \(\mathcal U_P(\pi)\) is the set of rewards for which
\(\pi\) is optimal under \(P\).
\end{lemma}

\begin{proof}
Optimality of a fixed policy is given by the closed inequalities
\[
A_{P,r}^{\pi}(s,a)\le0
\qquad(s\in S,\ a\in A(s)).
\]
Since policy values, and hence advantages, depend continuously on \(r\), any
limit of rewards in \(\mathcal U_P(\pi)\) makes \(\pi\) optimal.

Conversely, suppose \(\pi\) is optimal for \(r\), and for \(\varepsilon>0\)
define
\[
r_\varepsilon(s,a)
:=
r(s,a)+\varepsilon\mathbf 1_{\{a=\pi(s)\}}.
\]
Following \(\pi\) receives the additional reward \(\varepsilon\) at every
time step, so
\[
V_{P,r_\varepsilon}^{\pi}
=
V_{P,r}^{\pi}+\frac{\varepsilon}{1-\gamma}\one.
\]
For \(a\ne\pi(s)\), using \(P_{s,a}\one=1\),
\begin{align*}
A_{P,r_\varepsilon}^{\pi}(s,a)
&=
r(s,a)+\gamma P_{s,a}
\left(V_{P,r}^{\pi}+\frac{\varepsilon}{1-\gamma}\one\right)
-\left(V_{P,r}^{\pi}(s)+\frac{\varepsilon}{1-\gamma}\right)\\
&=
A_{P,r}^{\pi}(s,a)-\varepsilon
<0.
\end{align*}
Hence \(\pi(s)\) is the unique optimal action at every state under
\(r_\varepsilon\), so \(r_\varepsilon\in\mathcal U_P(\pi)\). Since
\(r_\varepsilon\to r\), the result follows.
\end{proof}

\begin{corollary}
\label{cor:strict-rewards-suffice}
The following are equivalent:
\begin{enumerate}[label=\textup{(\roman*)}]
\item \(P\equiv_{\mathrm{sa}}Q\);
\item \(\mathcal U_P(\pi)=\mathcal U_Q(\pi)\) for every deterministic
policy \(\pi\).
\end{enumerate}
\end{corollary}
\begin{proof}
The forward implication is immediate. For the converse, suppose
\[
\mathcal U_P(\pi)=\mathcal U_Q(\pi)
\]
for every deterministic policy \(\pi\). Taking closures gives
\[
\overline{\mathcal U_P(\pi)}
=
\overline{\mathcal U_Q(\pi)}.
\]
By Lemma~\ref{lem:strict-region-closure}, these closures are precisely the
sets of rewards for which \(\pi\) is optimal under \(P\) and \(Q\),
respectively. Thus each deterministic policy is optimal for exactly the same
rewards under the two kernels. An action is optimal at a state if and only
if it is selected there by some optimal deterministic policy, so
\(P\equiv_{\mathrm{sa}}Q\).
\end{proof}

\subsection{Composition Laws}

\begin{proposition}
\label{prop:group-action}
For all \(L,M\in\mathcal G\) for which the expressions below are transition
kernels,
\[
\Phi_I(P)=P,
\qquad
\Phi_M(\Phi_L(P))=\Phi_{ML}(P),
\qquad
\Phi_{L^{-1}}(\Phi_L(P))=P.
\]
\end{proposition}

\begin{proof}
Rewrite \eqref{eq:Phi} as
\[
\Phi_L(P)_{s,a}
=
P_{s,a}L^{-1}
+\frac1\gamma e_s^{\tr}(I-L^{-1}).
\]
Then
\begin{align*}
\Phi_M(\Phi_L(P))_{s,a}
&=\Phi_L(P)_{s,a}M^{-1}
  +\frac1\gamma e_s^{\tr}(I-M^{-1})\\
&=P_{s,a}L^{-1}M^{-1}
  +\frac1\gamma e_s^{\tr}
   \bigl((I-L^{-1})M^{-1}+I-M^{-1}\bigr)\\
&=P_{s,a}(ML)^{-1}
  +\frac1\gamma e_s^{\tr}\bigl(I-(ML)^{-1}\bigr)\\
&=\Phi_{ML}(P)_{s,a}.
\end{align*}
The identity and inverse formulas follow by setting \(L=I\) and
\(M=L^{-1}\).
\end{proof}
\enlargethispage{2\baselineskip}
\sectionmark{References}


\begin{thebibliography}{99}
\scriptsize
\setlength{\itemsep}{0pt}
\setlength{\parskip}{0pt}
\bibitem{puterman1994}
M. L. Puterman.
\newblock \emph{Markov Decision Processes: Discrete Stochastic Dynamic Programming}.
\newblock Wiley, 1994.

\bibitem{smallwood1966}
R. D. Smallwood.
\newblock Optimum policy regions for Markov processes with discounting.
\newblock \emph{Operations Research}, 14(4):658--669, 1966.

\bibitem{givan2003}
R. Givan, T. Dean, and M. Greig.
\newblock Equivalence notions and model minimization in Markov decision processes.
\newblock \emph{Artificial Intelligence}, 147(1--2):163--223, 2003.

\bibitem{ng1999}
A. Y. Ng, D. Harada, and S. Russell.
\newblock Policy invariance under reward transformations: Theory and application to reward shaping.
\newblock In \emph{Proceedings of the 16th International Conference on Machine Learning}, pages 278--287, 1999.

\bibitem{grimm2020}
C. Grimm, A. Barreto, S. Singh, and D. Silver.
\newblock The value equivalence principle for model-based reinforcement learning.
\newblock In \emph{Advances in Neural Information Processing Systems}, volume 33, 2020.

\bibitem{skalse2023}
J. M. V. Skalse, M. Farrugia-Roberts, S. Russell, A. Abate, and A. Gleave.
\newblock Invariance in policy optimisation and partial identifiability in reward learning.
\newblock In \emph{Proceedings of the 40th International Conference on Machine Learning}, volume 202 of PMLR, pages 32033--32058, 2023.

\bibitem{mustafin2025}
A. Mustafin, A. Pakharev, A. Olshevsky, and I. C. Paschalidis.
\newblock MDP geometry, normalization and reward balancing solvers.
\newblock In \emph{Proceedings of the 28th International Conference on Artificial Intelligence and Statistics}, volume 258 of PMLR, pages 2476--2484, 2025.

\bibitem{richens2025}
J. Richens, D. Abel, A. Bellot, and T. Everitt.
\newblock General agents contain world models.
\newblock arXiv:2506.01622, 2025.

\bibitem{letcher2026}
A. Letcher, M. Fellows, A. D. Goldie, J. Richens, J. N. Foerster, and O. Richardson.
\newblock Inverting the Bellman equation: From $Q$-values to world models.
\newblock arXiv:2606.21173, 2026.

\bibitem{ng2000}
A. Y. Ng and S. Russell.
\newblock Algorithms for inverse reinforcement learning.
\newblock In \emph{Proceedings of the 17th International Conference on
Machine Learning}, pages 663--670, 2000.

\bibitem{ghatrani2023}
Z. Ghatrani and A. Ghate.
\newblock Inverse Markov decision processes with unknown transition
probabilities.
\newblock \emph{IISE Transactions}, 55(6):588--601, 2023.

\bibitem{cao2021}
H. Cao, S. Cohen, and L. Szpruch.
\newblock Identifiability in inverse reinforcement learning.
\newblock In \emph{Advances in Neural Information Processing Systems},
volume 34, pages 12362--12373, 2021.

\bibitem{kim2021}
K. Kim, S. Garg, K. Shiragur, and S. Ermon.
\newblock Reward identification in inverse reinforcement learning.
\newblock In \emph{Proceedings of the 38th International Conference on
Machine Learning}, volume 139 of \emph{Proceedings of Machine Learning
Research}, pages 5496--5505, 2021.

\bibitem{ferns2004}
N. Ferns, P. Panangaden, and D. Precup.
\newblock Metrics for finite Markov decision processes.
\newblock In \emph{Proceedings of the 20th Conference on Uncertainty in
Artificial Intelligence}, pages 162--169, 2004.
\end{thebibliography}
\end{document}